\documentclass{article} 
\usepackage{iclr2027_conference,times}

\usepackage{hyperref}
\usepackage{url}

\usepackage{enumitem}
\usepackage{placeins}
\usepackage{wrapfig}
\usepackage{booktabs}
\usepackage{graphicx}
\usepackage[cmex10]{amsmath}
\usepackage{amssymb}
\usepackage{algorithmic}
\usepackage{algorithm}
\usepackage{amsthm}
\usepackage{tcolorbox}
\newtheorem{theorem}{Theorem}

\newtheorem{proposition}{Proposition}

\DeclareMathAlphabet\mathbfcal{OMS}{cmsy}{b}{n}
\newcommand{\ten}[1]{\mathbfcal{#1}}
\newcommand{\mat}[1]{\mathbf{#1}}
\newcommand{\vect}[1]{\boldsymbol{#1}}

\newcommand{\softmax}{\operatorname{softmax}}
\newcommand{\tuba}{\operatorname{TuBA}}

\title{Tucker Bottleneck Attention for Multi-Dimensional Sequence Modeling}

\author{Ryan Solgi 
\\
University of California-Santa Barbara\\
Santa Barbara, CA, USA\\
\texttt{solgi@ucsb.edu} \\
\And
Parsa Madinei\\
University of California-Santa Barbara\\
Santa Barbara, CA, USA\\
\texttt{madinei@ucsb.edu}\\
\And
Zheng Zhang\\
University of California-Santa Barbara\\
Santa Barbara, CA, USA\\
\texttt{zhengzhang@ece.ucsb.edu}
}

\iclrfinalcopy 
\begin{document}

\maketitle
\lhead{}
\begin{abstract} 

The quadratic cost of self-attention limits scalability to long sequences from multidimensional data. We introduce Tucker bottleneck attention (TuBA), which exploits low-rank tensor structure for efficient global token mixing. TuBA projects hidden tensors into compact Tucker cores, performs multi-head self-attention and linear projections on the cores, and writes updates back to the ambient space, enabling subquadratic computation. Its autoregressive extension combines bidirectional interactions within cores with causal attention across cores. On video prediction and global weather forecasting, TuBA achieves favorable accuracy--efficiency trade-offs over standard and efficient attention and task-specific models. Compared to standard self-attention, TuBA reduces error and computation by up to $24.7\%$ and $66.6\%$ for video prediction and $37.1\%$ and $85.1\%$ for autoregressive weather forecasting, with speedups up to $4.27\times$. Low-rank Tucker cores and multi-frame generation also outperform full-rank attention and frame-by-frame generation, respectively.
\end{abstract}


\section{Introduction}

Self-attention, the primary token-mixing mechanism in transformers~\citep{vaswani2017attention}, has driven major advances in machine learning~\citep{devlin2019bert,brown2020language,dosovitskiy2021image,nguyen2023climax}, but its quadratic cost limits scalability to long sequences. Consequently, extensive work has sought to improve attention efficiency~\citep{tay2022efficient}. Some methods restrict attention to local or sparse patterns~\citep{swin,child2019sparse,longformer,bigbird,axial}, while others replace softmax with kernel feature maps to avoid forming the attention matrix~\citep{linear,cosformer,performer,featuremapattention}, or compute attention using a smaller set of latent tokens~\citep{set,luna,perceiver}. Other approaches exploit low-rank attention structure. Linformer~\citep{linformer}, for example, projects keys and values into a low-dimensional space to achieve linear attention complexity, while related methods use low-rank reconstructions~\citep{nystrom,skyformer,scatterbrain,kronecker,tensorization}. These methods primarily rewrite or approximate self-attention itself; in contrast, exploiting low-rank tensor structure in the representations to reduce attention cost remains underexplored.

Modeling multidimensional data is central to scientific applications~\citep{kolda2009tensor}, from video prediction and world modeling~\citep{gao2022simvp,bruce2024genie} to weather and climate forecasting~\citep{earthformer,nguyen2023climax}. Such data exhibit interactions within and across modes that models can exploit~\citep{holmes2012turbulence,hannachi2007empirical,brunton2019data,solgi2025tensorized}. Existing architectures, however, face tradeoffs between efficiency and global token mixing. Recurrent models limit parallelization through sequential processing~\citep{shi2015convolutional,wang2017predrnn,wang2018predrnnpp,wang2019memory,chang2021mau}, while convolutional models rely on local receptive fields~\citep{gao2022simvp,tan2023temporal}. Likewise, local and factorized attention restrict direct interactions to local regions or individual axes~\citep{swin,earthformer,axial,timesformer,arnab2021vivit,predformer}. Yet multidimensional data can exhibit long-range, cross-mode interactions, motivating global token mixing~\citep{wallace1981teleconnections,wang2018nonlocal,lam2023graphcast}. Full attention provides global mixing but incurs quadratic cost and lacks structural priors~\citep{dosovitskiy2021image,dascoli2021convit}; moreover, its attention can become increasingly diffuse with sequence length~\citep{nakanishi2025scalable,velickovic2024softmax,jin2023training}. Efficient global attention~\citep{linear,linformer,nystrom} reduces computation but still lacks structural priors and can sacrifice accuracy~\citep{tay2022efficient}. We propose that the low-rank tensor structure of hidden representations provides an efficient inductive bias for global token mixing.


Transformer hidden representations exhibit approximate low-rank structure and lie on low-dimensional manifolds~\citep{yu2023lowrank,modell2025manifold}. For high-order representations, this structure can be exploited through tensor decomposition. Using Tucker factorization, we observe that hidden representations largely share a common subspace across samples: shared mode-wise factors with compact per-sample cores achieve nearly the same reconstruction accuracy as sample-specific decompositions (Appendix~\ref{sec:shared-subspace}). This observation motivates \emph{Tucker bottleneck attention (TuBA)}, which projects a multidimensional hidden tensor into a compact Tucker core using learned mode-wise read factors, applies standard multi-head self-attention (MHSA) within the core, and writes the update back to the original tokens (Figure~\ref{fig:tuba-overview}). TuBA thereby enables global token mixing at reduced computational cost.

Our main contributions are summarized below:

\begin{itemize}[leftmargin=*, itemsep=0pt, topsep=0pt, parsep=0pt, partopsep=0pt]
    \item \textbf{A high-order tensor view for efficient attention.}
    We introduce TuBA, which exploits low-rank tensor
    representations to perform global token mixing in a compact Tucker
    core at subquadratic cost. We further extend TuBA to autoregressive
    tensor generation through causal attention.

    \item \textbf{Interpretation and expressivity of TuBA.}
    We show that TuBA's learned read projection yields a unique least-squares Tucker core for associated factors, while its write projection produces a Tucker-structured update in a potentially distinct output subspace. TuBA also contains standard self-attention as a special case.
\item \textbf{Experimental results.} 
Across video and weather benchmarks, TuBA achieves favorable accuracy--efficiency trade-offs over standard and efficient attention methods and task-specific models. Compared with standard self-attention, TuBA reduces video prediction MSE by up to $24.7\%$ and FLOPs by up to $66.6\%$, while reducing autoregressive weather normalized RMSE by $37.1\%$ with a $4.27\times$ speedup. Ablations further show that intermediate Tucker ranks and multi-frame generation outperform full-rank attention and frame-by-frame generation, respectively.

\end{itemize}

\section{Background}
\label{sec:background}

\subsection{Multi-head self-attention}
\label{sec:bg-attn}

Let $\mat{X}\in\mathbb{R}^{N\times d}$ contain $N$ tokens of hidden dimension $d$.
Multi-head self-attention~\citep{vaswani2017attention} with $H$ heads of width $d_h$ is
\begin{equation}
\label{eq:attn}
\mathrm{MHSA}(\mat{X})
=\sum_{i=1}^{H}
  \underbrace{\softmax\!\Big(\tfrac{1}{\sqrt{d_h}}\,
    \mat{X}\mat{W}_Q^i\big(\mat{W}_K^i\big)^{\!\top}\mat{X}^{\!\top}\Big)}
    _{\textstyle \mat{S}_i(\mat{X})\,\in\,\mathbb{R}^{N\times N}}
  \;\mat{X}\;
  \underbrace{\mat{W}_V^i\mat{W}_O^i}
    _{\textstyle \mat{M}_i\,\in\,\mathbb{R}^{d\times d}},
\end{equation}

where $\softmax$ acts row-wise,
$\mat{W}_Q^i,\mat{W}_K^i,\mat{W}_V^i\in\mathbb{R}^{d\times d_h}$ are the learned query, key, and value projection matrices, and $\mat{W}_O^i\in\mathbb{R}^{d_h\times d}$ is the output-projection block for head $i$. Thus, each head applies the state-dependent operator $\mat{S}_i(\mat{X})$ along the token dimension and the fixed linear operator $\mat{M}_i$ along the channel dimension.

\begin{figure}[t]
    \centering
    \includegraphics[width=\textwidth]{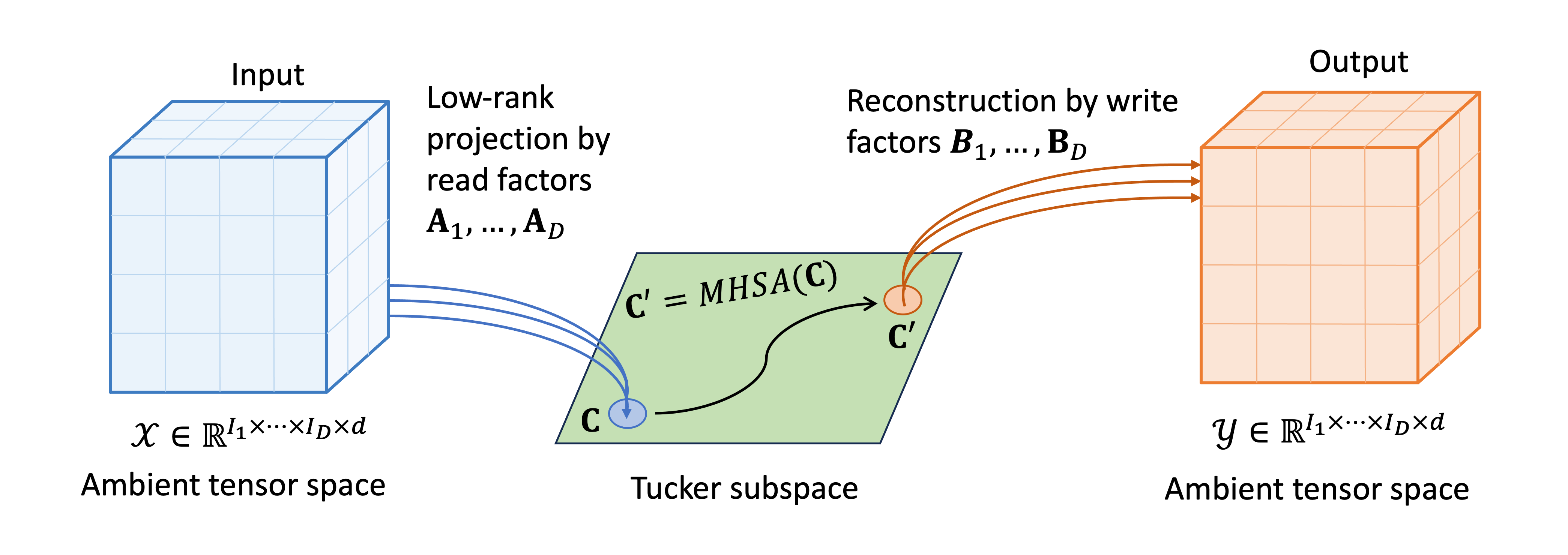}
    \caption{
    Overview of Tucker bottleneck attention. TuBA projects a tensor into a compact Tucker core, applies attention on the core, then reconstructs the output.
    }
    \label{fig:tuba-overview}
\end{figure}

\subsection{Tucker representation}
\label{sec:bg-tucker}
Let a rank-$(r_1,\dots,r_D)$ least-squares Tucker approximation~\citep{tucker1966}
of a tensor $\ten{X}\in\mathbb{R}^{I_1\times\cdots\times I_D}$ be
\begin{equation}
\label{eq:tucker}
\ten{X}\;\approx\;\hat{\ten{X}}
=
\ten{C}\times_1\mat{U}_1
\times_2\cdots\times_D\mat{U}_D,
\end{equation}
where $\ten{C}\in\mathbb{R}^{r_1\times\cdots\times r_D}$ is the core,
$\mat{U}_n\in\mathbb{R}^{I_n\times r_n}$, $n=1,\ldots,D$, are
full-column-rank factor matrices, and $\times_n$ denotes
multiplication along mode $n$.
For the given factor matrices, the core is uniquely obtained as
\begin{equation}
\label{eq:tucker-core}
\ten{C}
=
\ten{X}
\times_1\bar{\mat{U}}_1^{\top}
\times_2\cdots\times_D\bar{\mat{U}}_D^{\top},
\end{equation}
where $\bar{\mat{U}}_n^{\top}=\mat{U}_n^\dagger$ is the Moore--Penrose
pseudoinverse of $\mat{U}_n$, and $\bar{\mat{U}}_n=\mat{U}_n$ for orthonormal factors. 

\section{Method}
\label{sec:method}

We first introduce the TuBA layer and establish its connection to Tucker representations. We then examine its induced token mixing, extend it to autoregressive tensor generation, and describe its computational cost and integration into a transformer.

\subsection{Tucker Bottleneck Attention}
\label{sec:tuba-layer}

Tucker bottleneck attention (TuBA) performs global token mixing using a Tucker core rather than across all input tokens. It projects a
multidimensional hidden tensor into a compact core using learned
mode-wise factor matrices, applies standard multi-head self-attention
to the core, and maps the resulting update back to the original
tensor. Let $\ten{X}\in\mathbb{R}^{I_1\times\cdots\times I_D\times d}$ be a hidden tensor, where the first $D$ modes index the native token axes and $d$ is the hidden dimension. For example, in spatiotemporal prediction, after the inputs are tokenized, the hidden representation is $\ten{X}\in\mathbb{R}^{F\times P\times Q\times d}$ with $D=3$, where $F$ indexes frames, $P$ and $Q$ index the spatial token grid, and each token carries a $d$-dimensional feature vector.
For ranks $(r_1,\ldots,r_D)$ with $r_n\leq I_n$,
TuBA learns read and write factors
$\mat{A}_n,\mat{B}_n\in\mathbb{R}^{I_n\times r_n}$ for each token axis
$n=1,\ldots,D$. It first reads the hidden representation into a compact
Tucker core:
\begin{equation}
\label{eq:tuba-read}
\ten{C}
=
\ten{X}
\times_1\mat{A}_1^\top
\times_2\cdots
\times_D\mat{A}_D^\top
\in
\mathbb{R}^{r_1\times\cdots\times r_D\times d}.
\end{equation}
Because the projection acts mode-wise, the core is a compact representation that preserves the multidimensional structure of the input, with each mode's resolution set by its own rank. Flattening the first $D$ modes of $\ten{C}$ gives
$\mat{C}\in\mathbb{R}^{R\times d}$, where
$R=\prod_{n=1}^{D}r_n$ is the number of core slots. TuBA applies
standard multi-head self-attention to these slots:
\begin{equation}
\label{eq:tuba-attend}
\mat{C}'=\mathrm{MHSA}(\mat{C})
\in\mathbb{R}^{R\times d}.
\end{equation}
Since non-causal MHSA is permutation equivariant, and any core positional
encoding is indexed by core coordinates, the flattening order does not
affect the result; hence, reshaping $\mat{C}'$ gives $\ten{C}'\in\mathbb{R}^{r_1\times\cdots\times r_D\times d}$. The output of the TuBA layer, $\ten{Y}$, is obtained by mapping this tensor back to the ambient tensor space through the learned write factors:
\begin{equation}
\label{eq:tuba-write}
\ten{Y}
=
\ten{C}'
\times_1\mat{B}_1
\times_2\cdots
\times_D\mat{B}_D
\in\mathbb{R}^{I_1\times\cdots\times I_D\times d}.
\end{equation}

TuBA contains standard MHSA as a full-rank special case: when
$r_n=I_n$ and $\mat{A}_n=\mat{B}_n=\mat{I}$ for every $n$, it
reduces exactly to MHSA over the $N$ flattened tokens. The following proposition connects the projection to a least-squares Tucker core and characterizes the structure of the attended update.

\begin{proposition}[Interpretation of TuBA]
\label{prop:tuba-tucker}
For full-column-rank read factors $\mat{A}_n$, define
\[
\mat{U}_n
=
\mat{A}_n(\mat{A}_n^\top\mat{A}_n)^{-1}.
\]
The TuBA read core in Eq.~\eqref{eq:tuba-read} is the unique least-squares Tucker core associated with the factors $\mat{U}_1,\ldots,\mat{U}_D$. Reconstruction through these factors gives the orthogonal projection of $\ten{X}$ onto the tensor-product subspace determined by the read factors. For any transformed core, the TuBA update in Eq.~\eqref{eq:tuba-write} has an exact Tucker representation with mode-$n$ rank at most $r_n$ for each token mode.
\end{proposition}

See
Appendix~\ref{sec:tuba-tucker-proof} for proof.

\subsection{TuBA token mixing}
\label{sec:tuba-analysis}

To clarify how core attention induces interactions among the original tokens and the distinct roles of the read and write factors, we express the read, attend, and write operations as a single token-mixing map. Let $\mat{X}\in\mathbb{R}^{N\times d}$ denote the unfolding of the first $D$ modes of $\ten{X}\in\mathbb{R}^{I_1\times\cdots\times I_D\times d}$, with the mode-1 index varying fastest and $N=\prod_{n=1}^{D}I_n$. Collecting the mode-wise factors gives the Kronecker-product read and write matrices
\begin{equation}
\label{eq:tuba-kron-bases}
\mat{A}
=
\mat{A}_D\otimes\cdots\otimes\mat{A}_1
\in\mathbb{R}^{N\times R},
\qquad
\mat{B}
=
\mat{B}_D\otimes\cdots\otimes\mat{B}_1
\in\mathbb{R}^{N\times R}.
\end{equation}
Equations~\eqref{eq:tuba-read}--\eqref{eq:tuba-write} then take the matrix form
\begin{equation}
\label{eq:tuba-flat}
\tuba(\mat{X}) = \mat{B}\,\mathrm{MHSA}(\mat{A}^{\top}\mat{X}).
\end{equation}
Expanding headwise gives
\begin{equation}
\label{eq:tuba-headwise}
\tuba(\mat{X})
=
\sum_{i=1}^{H}
\mat{B}\,
\widetilde{\mat{S}}_i(\mat{A}^{\top}\mat{X})
\mat{A}^{\top}\mat{X}\mat{M}_i,
\end{equation}
where $\widetilde{\mat{S}}_i(\mat{A}^{\top}\mat{X})\in\mathbb{R}^{R\times R}$
is the attention matrix of head $i$ on the Tucker core. Thus, TuBA replaces
each head's $N\times N$ attention computation with an $R\times R$ core
attention computation, retaining the standard channel-side maps $\mat{M}_i$.
The induced token-mixing map
$\mat{B}\widetilde{\mat{S}}_i(\mat{A}^{\top}\mat{X})\mat{A}^{\top}$
has rank at most $R$, with column and row spaces contained in
$\operatorname{range}(\mat{B})$ and $\operatorname{range}(\mat{A})$,
respectively. It is generally nonsymmetric, and its input- and output-side
low-rank subspaces need not coincide. TuBA therefore learns separate read
and write factors (the untied configuration). The tied special case,
$\mat{A}_n=\mat{B}_n$ for all $n$, assumes one shared Tucker subspace
suffices for reading the input and writing the attended update. Denote the
corresponding single-layer function classes by $\mathcal{F}_{\mathrm{untied}}$
and $\mathcal{F}_{\mathrm{tied}}$, respectively.

\begin{theorem}[Expressivity of untied read and write factors]
\label{thm:expressivity}
For $N>1$ and the same Tucker ranks, hidden dimension, and heads, with
MHSA as in Eq.~\eqref{eq:attn},
$\mathcal{F}_{\mathrm{tied}}\subsetneq\mathcal{F}_{\mathrm{untied}}$.
\end{theorem}

See Appendix~\ref{sec:tuba-expressivity-proof} for proof.

Theorem~\ref{thm:expressivity} shows that untied factors represent mappings that tied factors cannot at the same core size. Untying the factors increases expressivity without increasing TuBA's computational cost.

\subsection{Causal TuBA for Autoregressive Tensor Generation}
\label{sec:tuba-causal}

TuBA treats a tensor as a single computational unit. When a task
requires causal or, more generally, autoregressive generation, the
ordering is imposed over a sequence of tensors rather than being
inherent to any individual tensor mode. Depending on the task, the choice of tensor determines the granularity
at which causality is imposed. For example, a video clip may be processed non-causally as a single three-dimensional tensor, whereas causal generation can treat individual frames or contiguous blocks of frames as an ordered sequence of tensors. Let $\ten{X}_1,\ldots,\ten{X}_T
\in\mathbb{R}^{I_1\times\cdots\times I_D\times d}$ be a sequence
of tensor observations. TuBA independently reads each $\ten{X}_t$
into a core $\mat{C}_t\in\mathbb{R}^{R\times d}$ and applies masked attention to their concatenation
\[
\mat{C}_{1:T}
=
\begin{bmatrix}
\mat{C}_1^\top & \cdots & \mat{C}_T^\top
\end{bmatrix}^{\!\top}
\in\mathbb{R}^{TR\times d}.
\]
For core slots $a,b\in\{1,\ldots,R\}$ at times $t$ and $s$, respectively,
the causal mask is
\begin{equation}
\label{eq:tuba-block-causal-mask}
\boldsymbol{\Lambda}_{(t,a),(s,b)}
=
\begin{cases}
0, & s\leq t,\\
-\infty, & s>t.
\end{cases}
\end{equation}
The mask is added to the pre-softmax attention scores in
Eq.~\eqref{eq:attn}. The usual triangular mask becomes block lower-triangular, with $R\times R$ blocks: all core slots within a tensor may interact, while each tensor attends only to itself and earlier tensors. Block-causal masking of this form has been used over raw token grids~\citep{blockdiffusion, causvid, blockgpt}; TuBA instead imposes it on low-rank Tucker cores to autoregressively generate a sequence of tensors. This construction allows the model to jointly predict the elements of the next tensor while conditioning on preceding tensors. The tensor size thus determines the granularity of each generation step, while the Tucker ranks determine the size of the representation used for attention.

\subsection{Computational, memory, and parameter cost}
\label{sec:tuba-cost}


Tucker core attention costs $O(R^2d+Rd^2)$ FLOPs, while the mode-wise read and write contractions are bounded by $O(N\bar{R}d)$, where $\bar{R}=\sum_{n=1}^{D}r_n$. Hence, TuBA's computational complexity is $O\!\left(N\bar{R}d+R^2d+Rd^2\right)$. Since $\bar{R}\leq DR$, for fixed $D$, this complexity is asymptotically lower than the $O(N^2d+Nd^2)$ cost of standard self-attention whenever $R=o(N)$, i.e., $R/N\to0$ as $N\to\infty$. TuBA requires $O(Nd+HR^2)$ activation memory, compared with $O(Nd+HN^2)$ for standard self-attention when attention matrices are explicitly stored; with FlashAttention~\citep{dao2022flashattention}, both require $O(Nd)$, while TuBA retains its computational advantage (see Appendix~\ref{sec:tuba-cost-details}). In terms of parameters, the read and write factors add $2\sum_{n=1}^{D}I_nr_n$ per layer, which grows with the sum of the mode sizes rather than with the number of tokens $N=\prod_{n}I_n$; hence, TuBA's parameter overhead remains negligible as the multidimensional input size grows.

\subsection{TuBA Transformer block}
\label{sec:tuba-block}

TuBA replaces the multi-head self-attention sublayer of a standard
pre-norm transformer block; all remaining components are unchanged.
Given an input hidden tensor $\ten{H}$, the TuBA sublayer is
\begin{equation}
\label{eq:tuba-attention-block}
\ten{Z}
=
\ten{H}
+
\tuba\!\left(\mathrm{LN}(\ten{H})\right),
\end{equation}
where $\mathrm{LN}$ applies layer normalization along the hidden
dimension. The resulting tensor then passes through a standard
feed-forward sublayer:
\begin{equation}
\label{eq:tuba-ffn-block}
\ten{H}_{\mathrm{out}}
=
\ten{Z}
+
\mathrm{FFN}\!\left(\mathrm{LN}(\ten{Z})\right).
\end{equation}
TuBA changes only global token mixing, while normalization, residual
connections, and channel-wise feed-forward computation remain those of a standard transformer block. The Tucker bottleneck constrains the attention update, while the residual path carries the incoming representation forward at its original resolution. Consequently, the rank bounds on the TuBA update do not impose the same bounds on the hidden representations passing through the network.

\subsection{Training and initialization}
\label{sec:tuba-training}

All TuBA parameters, including the read and write factors
$\{\mat{A}_n,\mat{B}_n\}_{n=1}^{D}$, are
optimized with standard backpropagation. The mode-wise
contractions are differentiable, so TuBA requires neither alternating
tensor-factor optimization nor a specialized training objective. The Tucker ranks $(r_1,\ldots,r_D)$ are considered to be hyperparameters and determined empirically. We initialize each $\mat{A}_n$ with orthonormal columns and set
$\mat{B}_n\leftarrow\mat{A}_n$. The two factor sets are subsequently optimized independently. Thus, training begins from a shared Tucker basis while allowing the read and write subspaces to specialize to the task.

\subsection{Comparison with prior work}

Structured token interaction methods restrict attention to sparse, windowed, or axis-wise patterns~\citep{child2019sparse,longformer,bigbird,swin,axial,timesformer}, while reduction methods aggregate, select, or prune representations~\citep{wang2022pvtv2,rao2021dynamicvit,madinei2026interlace}. Their effectiveness depends on the chosen connectivity~\citep{swin,earthformer} and composition of attention operations~\citep{dong2022cswin,tarasiou2023vits,predformer}. TuBA instead exploits low-rank tensor structure, learning mode-wise factors to perform global token mixing in a compact Tucker subspace. Prior low-rank~\citep{linformer} and latent-token methods~\citep{set,luna,perceiver} compress one side of attention while retaining full-length queries, whereas other methods approximate attention scores while retaining full-length representations~\citep{nystrom, skyformer, scatterbrain}. Tensorized attention~\citep{kronecker,tensorization} similarly exploits Kronecker structure in the attention matrix without compressing hidden representations. TuBA instead compresses the hidden tensor itself, performing attention and all query, key, value, and output projections within its Tucker core. Unlike Tucker-based model compression~\citep{klein2026tuckerattention, gu2025tensorllm}, which targets model weights, TuBA compresses hidden representations along token dimensions to reduce attention and projection costs.

\section{Experiments}



We evaluate TuBA's accuracy and efficiency on video prediction and weather forecasting. On both tasks, we compare TuBA under a controlled setup with standard self-attention, efficient global attention methods---Linformer~\citep{linformer}, Nystr\"{o}m attention~\citep{nystrom}, and linear attention~\citep{linear}---and the convolutional mixer SimVP~\citep{gao2022simvp}. For video prediction, we further compare with recurrent, convolutional, local-attention, MLP, and pooling-based predictors. For weather forecasting, we compare with Earthformer~\citep{earthformer} and autoregressive FourCastNet~\citep{fourcastnet}, and evaluate causal TuBA for autoregressive forecasting. We also study the effects of Tucker core size and causal granularity on accuracy and efficiency. Bold and underlined values denote the best and second-best results, respectively. All experiments use a single NVIDIA H100 GPU, with implementation details in Appendix~\ref{sec:implement-details}.

\begin{table*}[t]
\centering
\small
\caption{Comparison of predictive accuracy and computational efficiency
on Human3.6M.}
\label{tab:human36m}
\begin{tabular}{@{}l rrrrr rrrr@{}}
\toprule
Method
  & \multicolumn{5}{c}{Accuracy}
  & \multicolumn{4}{c}{Efficiency} \\
\cmidrule(lr){2-6} \cmidrule(lr){7-10}
  & MSE$\downarrow$ & MAE$\downarrow$
  & SSIM$\uparrow$ & PSNR$\uparrow$ & LPIPS$\downarrow$
  & Params & FLOPs & FPS & FPS$^*$ \\
\midrule
TuBA
  & \textbf{113.21} & \textbf{1499.3}
  & 0.9664 & \textbf{33.79} & \underline{0.0339}
  & 11.1M & \textbf{63.5G} & \underline{580} & \underline{739} \\
SimVP
  & \underline{116.91} & \underline{1525.9}
  & \textbf{0.9824} & \underline{33.72} & 0.0340
  & 11.1M & 74.0G & \textbf{589} & 576 \\
ViT
  & 141.64 & 1642.6
  & \underline{0.9815} & 33.50 & \textbf{0.0335}
  & 11.1M & 172.8G & 151 & 694 \\
Linear
  & 149.79 & 1639.4
  & 0.9615 & 32.54 & 0.0397
  & 11.1M & \underline{70.6G} & 555 & \textbf{741} \\
Linformer
  & 277.30 & 2107.8
  & 0.9501 & 30.02 & 0.0406
  & 23.7M & 76.2G & 552 & 701 \\
Nystr\"{o}m
  & 196.01 & 1809.8
  & 0.9575 & 31.45 & 0.0439
  & 11.1M & 75.2G & 473 & 537 \\
\bottomrule
\end{tabular}
\end{table*}

\begin{table*}[t]
\centering
\small
\caption{Comparison of predictive accuracy and computational efficiency
on KittiCaltech.}
\label{tab:kitticaltech}
\begin{tabular}{@{}l rrrrr rrrr@{}}
\toprule
Method
  & \multicolumn{5}{c}{Accuracy}
  & \multicolumn{4}{c}{Efficiency} \\
\cmidrule(lr){2-6} \cmidrule(lr){7-10}
  & MSE$\downarrow$ & MAE$\downarrow$
  & SSIM$\uparrow$ & PSNR$\uparrow$ & LPIPS$\downarrow$
  & Params & FLOPs & FPS & FPS$^*$ \\
\midrule
TuBA
  & \textbf{156.07} & \textbf{1641.7}
  & \underline{0.8963} & \textbf{27.25} & \underline{0.0772}
  & 7.0M & 29.2G & 142 & 382 \\
SimVP
  & \underline{156.94} & \underline{1683.9}
  & \textbf{0.9019} & \underline{26.87} & \textbf{0.0676}
  & 8.6M & 60.6G & 218 & 197 \\
ViT
  & 207.33 & 1970.4
  & 0.8621 & 25.85 & 0.1159
  & 7.0M & 87.5G & 62 & \underline{392} \\
Linear
  & 174.18 & 1732.0
  & 0.8782 & 26.70 & 0.1026
  & 7.0M & \textbf{25.2G} & \textbf{293} & \textbf{417} \\
Linformer
  & 285.74 & 2303.4
  & 0.8337 & 24.29 & 0.1346
  & 16.8M & 29.6G & \underline{287} & 381 \\
Nystr\"{o}m
  & 206.80 & 1972.4
  & 0.8672 & 25.87 & 0.1047
  & 7.0M & \underline{29.0G}
  & 216 & 185 \\
\bottomrule
\end{tabular}
\end{table*}

\subsection{Video Prediction}
\label{sec:exp-video}


For video prediction, we use two complementary benchmarks. Human3.6M captures structured human motion, where models predict four future frames from four observed frames at $256{\times}256$ resolution. KittiCaltech evaluates cross-dataset generalization to real-world driving scenes with camera motion, occlusions, and clutter: models trained on KITTI predict one future frame from ten observed frames on Caltech Pedestrian at $128{\times}160$ resolution. For controlled comparisons among attention mechanisms, all variants use the same shallow convolutional encoder--decoder and MLP, differing only in their attention blocks. SimVP~\citep{gao2022simvp} uses the same encoder--decoder but replaces the transformer with its original Inception-UNet (IncepU) mixer. This setup isolates the effect of the token-mixing mechanism while keeping the remaining architecture fixed.

\paragraph{Prediction accuracy and computational efficiency.}
Tables~\ref{tab:human36m} and~\ref{tab:kitticaltech} show that TuBA achieves the lowest MSE and MAE on both benchmarks, where FPS and FPS$^*$ denote \texttt{fp32} and mixed-precision inference, respectively. It also achieves the highest PSNR and second-best LPIPS. Compared with ViT using standard self-attention, TuBA reduces MSE by up to $24.7\%$ and FLOPs by up to $66.6\%$, with $2.3$--$3.8\times$ faster \texttt{fp32} inference and $3.3$--$4.3\times$ faster training (Appendix~\ref{sec:memory-walltime}), while achieving comparable mixed-precision throughput. Under mixed precision, the gap narrows as ViT uses an optimized attention kernel, while TuBA's additional read and write operations lack specialized kernels. The MSE reduction is larger on KittiCaltech ($24.7\%$ versus $20.1\%$), demonstrating effectiveness in real-world scenes with camera motion and clutter. TuBA uses no more parameters than ViT or Linformer and reduces MSE over efficient-attention baselines by $10$--$59\%$ at comparable or lower FLOPs. It also slightly outperforms SimVP in MSE with fewer FLOPs and $1.3$--$1.9\times$ higher mixed-precision throughput. Despite attending over fewer core slots than ViT has tokens, TuBA achieves higher accuracy, suggesting that its low-rank tensor bottleneck provides a beneficial inductive bias rather than merely a lossy approximation.

\begin{figure}[t]
    \centering
    \includegraphics[width=\textwidth]{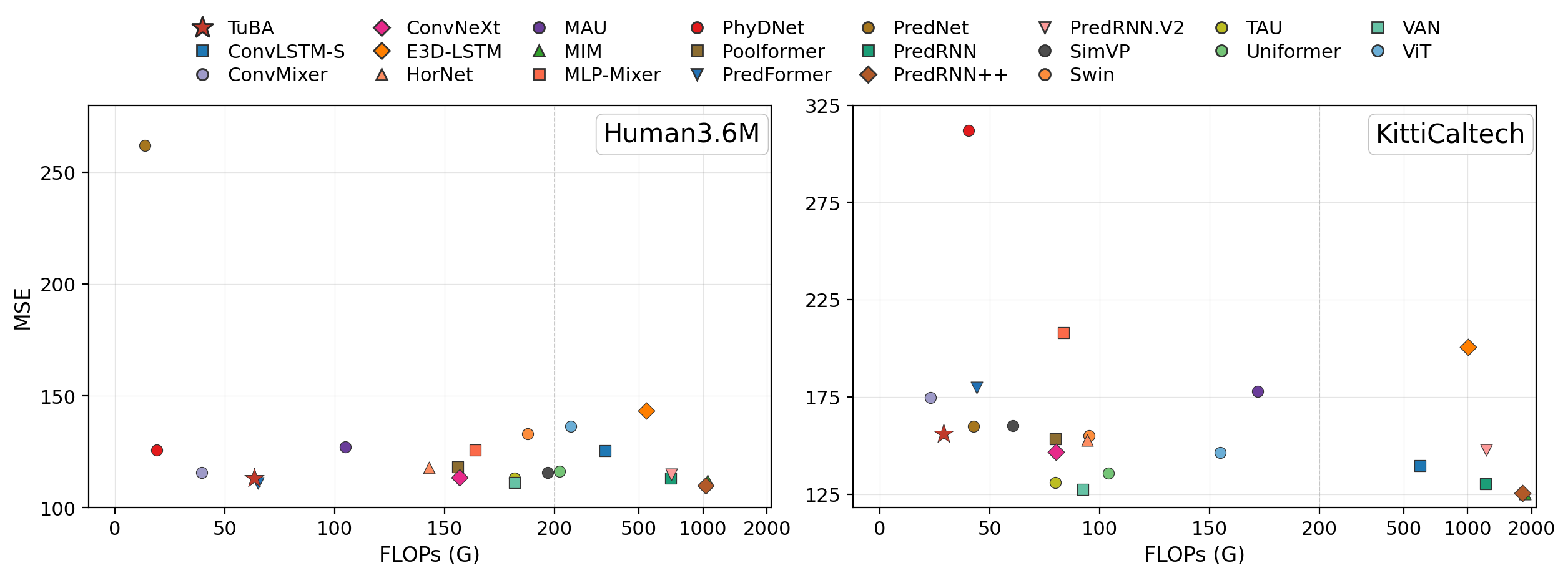}
    \vspace{-15pt}
    \caption{Comparison of prediction error (MSE) and computational cost (FLOPs) on Human3.6M and KittiCaltech; the FLOPs axis is compressed above 200G. TuBA lies on the FLOPs--MSE Pareto frontier on both benchmarks.}
    \label{fig:flops-mse}
\end{figure}

\paragraph{Comparison to other predictors.}
Figure~\ref{fig:flops-mse} compares TuBA with other video prediction methods in terms of prediction accuracy and computational demand. Since TuBA is trained with the standard MSE loss, we use MSE as the primary accuracy metric and FLOPs as a representative measure of computational cost. On both Human3.6M and KittiCaltech, TuBA lies on the FLOPs--MSE Pareto frontier: none of the compared methods achieves both lower MSE and lower FLOPs. TuBA therefore provides a favorable accuracy--efficiency trade-off, attaining low prediction error without the substantially larger computational cost of many recurrent and transformer-based predictors. These results provide broader context rather than a strictly controlled comparison, as methods differ in their training losses, architectures, skip connections, and input-processing pipelines.

\paragraph{Full-rank attention is not optimal.}
\begin{wrapfigure}{R}{0.5\textwidth}
    \centering
    \vspace{-8pt}
    \includegraphics[width=0.5\textwidth]{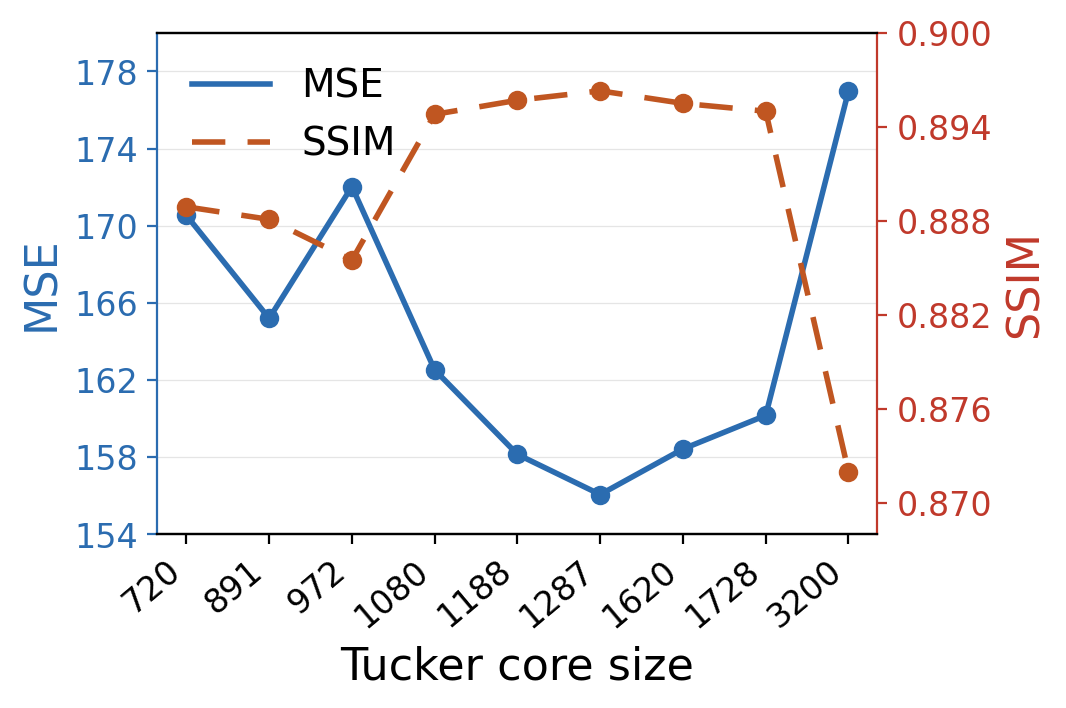}
    \caption{Effect of Tucker core size on TuBA's prediction accuracy on KittiCaltech, measured by MSE and SSIM.}
    \label{fig:rank-abl}
\end{wrapfigure}
Figure~\ref{fig:rank-abl} shows the effect of Tucker rank on model accuracy by varying the Tucker core sizes from $R=720$ to the full-rank setting ($R=3200$), which matches the number of tokens used by the standard self-attention (Appendix~\ref{sec:rank-ablation}).
We observe that accuracy does not improve monotonically as the core size increases, despite the additional computation. The best results occur at the intermediate core size $R=1287$, which attains the lowest MSE and the highest SSIM. Increasing the core beyond this point degrades prediction accuracy, and the full-rank core performs substantially worse. We also observe the same pattern comparing the standard ViT with TuBA in Tables~\ref{tab:human36m} and~\ref{tab:kitticaltech}. The fact that performance is maximized at an intermediate low rank rather than in the full-rank or standard self-attention
supports TuBA's premise that global multidimensional interactions can be modeled effectively in a compact low-rank tensor representation.

\subsection{Weather Prediction}
\label{sec:exp-weather}
We further evaluate TuBA on global weather forecasting to test its ability to model multivariate atmospheric dynamics. We use ERA5~\citep{hersbach2020era5} at $1.40625^\circ$ resolution ($128{\times}256$) with 25 atmospheric and surface variables, including geopotential, temperature, and wind at multiple pressure levels, sampled at six-hour intervals. Models observe 12 frames (three days) and predict the next 12 frames, either in one shot or autoregressively; we further evaluate autoregressive models over horizons up to 10 days (40 frames). For TuBA and the attention baselines, we use direct spatial patch embedding and a linear output head, isolating the ability to learn atmospheric dynamics without convolutional feature extraction. Autoregressive TuBA uses the block-causal formulation of Section~\ref{sec:tuba-causal}. We also compare with SimVP and two Earth-system forecasting models: Earthformer~\citep{earthformer}, a one-shot predictor combining a hierarchical convolutional encoder--decoder with cuboid attention, and FourCastNet~\citep{fourcastnet}, an autoregressive model that uses adaptive Fourier neural operators (AFNO) for spatial mixing.

\begin{wrapfigure}{R}{0.5\textwidth}
    \centering
    \vspace{-8pt}
    \includegraphics[width=0.48\textwidth]{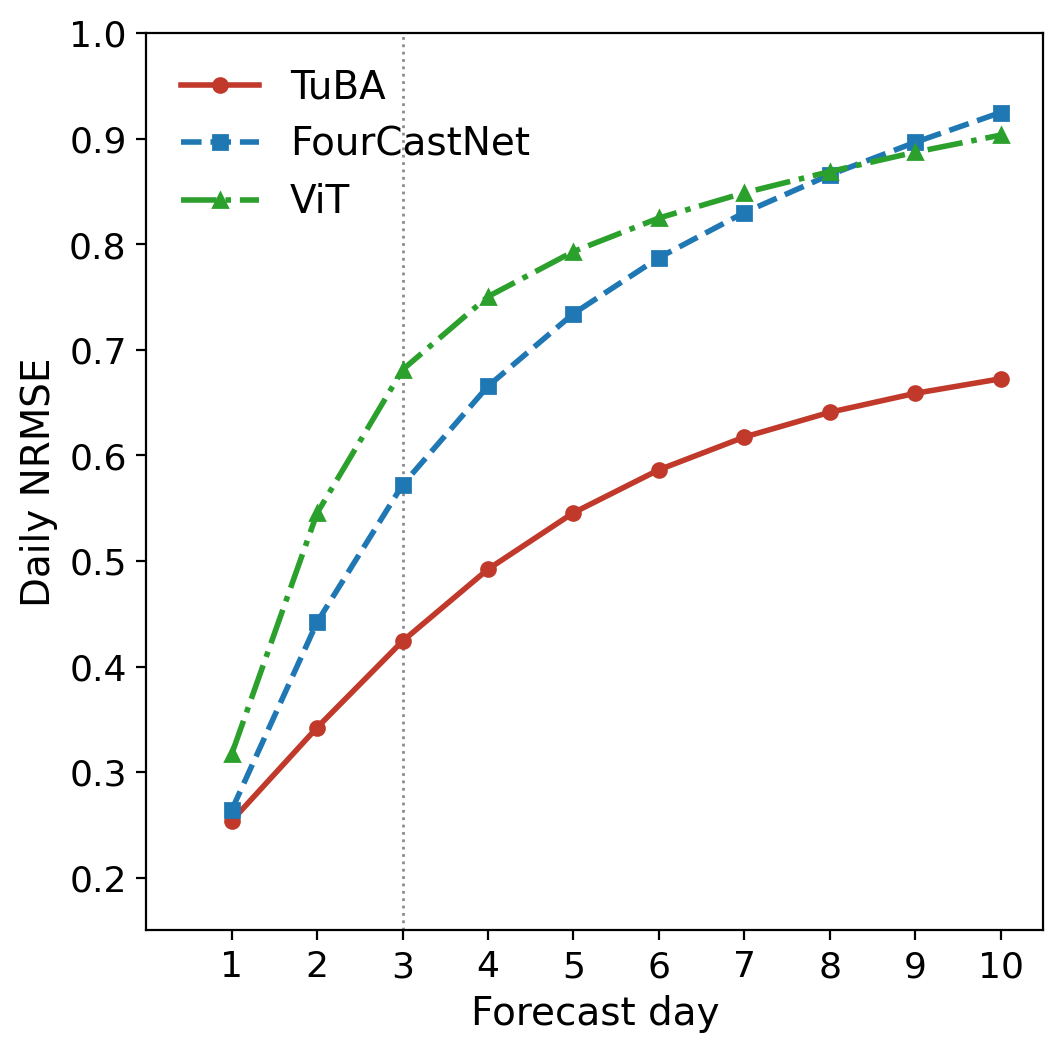}
    \caption{Daily NRMSE of autoregressive forecasts over ten days.
    The dotted line marks the three-day forecast horizon used
    during training.}
    \label{fig:era5-long-horizon}
\end{wrapfigure}

\begin{table*}[t]
\centering
\small
\caption{Comparison of predictive accuracy and computational efficiency
on ERA5. Models are trained on 2011--2015 and evaluated on 2017--2018.
(AR) denotes autoregressive generation.}
\label{tab:era5}
\begin{tabular}{@{}l r rrr rr rr r@{}}
\toprule
Method
& NRMSE$\downarrow$
& z\textsubscript{500}$\downarrow$
& t\textsubscript{850}$\downarrow$
& t\textsubscript{2m}$\downarrow$
& u\textsubscript{10}$\downarrow$
& v\textsubscript{10}$\downarrow$
& Params
& FLOPs
& clips/s \\
\midrule
TuBA 
  & 0.399
  & 0.141 & 0.162 & 0.126
  & 0.536 & 0.627
  & 8.72M & \textbf{40.8G}
  & \textbf{600} \\
SimVP
  & \underline{0.373}
  & \underline{0.127} & \underline{0.158} & \underline{0.122}
  & \underline{0.484} & \underline{0.572}
  & 8.60M & 67.8G
  & 192 \\
Earthformer
  & 0.415
  & 0.148 & 0.182 & 0.194
  & 0.526 & 0.613
  & 8.66M & 86.5G
  & 191 \\
ViT
  & 0.492
  & 0.212 & 0.204 & 0.147
  & 0.671 & 0.793
  & 8.71M & 198G
  & 362 \\
Linformer 
  & 0.516
  & 0.216 & 0.212 & 0.147
  & 0.720 & 0.866
  & 33.87M & 56.7G
  & 434 \\
Linear
  & 0.550
  & 0.259 & 0.242 & 0.169
  & 0.731 & 0.879
  & 8.71M & \underline{45.4G}
  & \underline{514} \\
Nystr\"{o}m 
  & 0.552
  & 0.259 & 0.243 & 0.172
  & 0.731 & 0.883
  & 8.71M & 48.0G
  & 401 \\
\midrule
TuBA (AR) 
  & \textbf{0.349}
  & \textbf{0.113} & \textbf{0.136} & \textbf{0.107}
  & \textbf{0.458} & \textbf{0.534}
  & 7.67M & 67.2G
  & 265 \\
FourCastNet (AR)
  & 0.442
  & 0.153 & 0.167 & 0.123
  & 0.559 & 0.659
  & 8.80M & 54.4G
  & 64 \\
ViT (AR)
  & 0.555
  & 0.203 & 0.203 & 0.184
  & 0.652 & 0.767
  & 7.26M & 451.2G
  & 62 \\
\bottomrule
\end{tabular}
\vspace{-10pt}
\end{table*}


\paragraph{Prediction accuracy and computational efficiency.}
Table~\ref{tab:era5} reports the mean normalized RMSE (NRMSE) across all 25 channels and
individual NRMSE values for five key weather fields. One-shot TuBA reduces the mean NRMSE of the standard ViT by $18.9\%$ and that of the efficient-attention baselines by $22.7$--$27.7\%$, with improvements across all five individually reported fields. As in video prediction, the improvement over the standard self-attention supports the benefit of TuBA's structured low-rank inductive bias for learning spatiotemporal dynamics. These accuracy gains accompany the lowest FLOPs and highest forecast throughput among all methods in Table~\ref{tab:era5}. TuBA requires $79.4\%$ fewer FLOPs than ViT while delivering $1.66\times$ higher throughput. It also uses fewer FLOPs and achieves higher throughput than all three efficient attention baselines. Compared with Earthformer, TuBA achieves lower mean NRMSE, although Earthformer is slightly more accurate on the two surface wind components. TuBA nevertheless delivers $3.14\times$ higher throughput
with $52.8\%$ fewer FLOPs. One-shot TuBA has a $7.0\%$ higher mean NRMSE than SimVP, but requires $39.8\%$ fewer FLOPs and delivers $3.13\times$ higher throughput; the autoregressive formulation below closes this gap and surpasses SimVP. By contrast, the efficient-attention baselines, whose throughput is closest to TuBA's, have $38.3$--$48.0\%$ higher mean NRMSE than SimVP. Even in its one-shot setting, TuBA outperforms autoregressive FourCastNet, achieving $9.7\%$ lower mean NRMSE and $9.38\times$ higher throughput.

\paragraph{Autoregressive TuBA.}

We additionally evaluate the causal formulation of TuBA from
Section~\ref{sec:tuba-causal} for autoregressive forecasting.
Each tensor contains four consecutive frames; attention is
bidirectional within its Tucker core and causal across tensor cores. Autoregressive TuBA achieves the lowest mean NRMSE and the lowest error on every individual field reported in Table~\ref{tab:era5}. It surpasses both SimVP and Earthformer in accuracy while requiring fewer FLOPs and achieving higher throughput, and reduces mean NRMSE by $12.5\%$ relative to one-shot TuBA. Compared with autoregressive ViT, it achieves $37.1\%$ lower mean NRMSE with $85.1\%$ fewer FLOPs and $4.27\times$ higher throughput. Compared with FourCastNet, autoregressive TuBA reduces mean NRMSE by $21.0\%$ and achieves $4.14\times$ higher throughput despite requiring more FLOPs. This throughput advantage reflects both TuBA's fewer autoregressive steps and differences in hardware
efficiency between attention and AFNO's Fourier-domain operations. We further extend autoregressive forecasting to ten days. Figure~\ref{fig:era5-long-horizon} shows that TuBA maintains lower daily NRMSE than both FourCastNet and ViT throughout the forecast period. Its advantage persists beyond the three-day training horizon, with a substantial separation from both baselines through day ten. These results demonstrate autoregressive TuBA's accuracy and efficiency across short- and medium-range weather forecast horizons.



\paragraph{Frame-by-frame generation is not optimal.}
We study how the granularity of causality affects autoregressive TuBA by varying the number of frames in each tensor. As shown in Table~\ref{tab:gen-g}, generating multiple frames jointly improves accuracy over frame-by-frame generation. The lowest NRMSE occurs with four-frame tensors; increasing the tensor size to six frames further improves computational efficiency, but lowers accuracy. Across the tested granularities, the parameter count increases modestly with tensor size. These results show that the choice of tensor size controls the trade-off between forecasting accuracy and computational efficiency, with the best accuracy achieved at an intermediate granularity. 

\begin{wraptable}{R}{0.5\textwidth}
\centering
\small
\vspace{-\intextsep}
\caption{Effect of causal granularity (frames per tensor) on
autoregressive TuBA's accuracy and computational efficiency on ERA5.}
\label{tab:gen-g}
\begin{tabular}{@{}c c cc c@{}}
\toprule
Granularity
  & NRMSE$\downarrow$
  & Params & FLOPs
  & clips/s\\
\midrule
1 & 0.426
  & \textbf{7.27M} & 89.3G
  & 100 \\
2 & \underline{0.358}
  & \underline{7.41M} & 85.2G
  & 163 \\
4 & \textbf{0.349}
  & 7.67M & \underline{67.2G}
  & \underline{265} \\
6 & 0.368
  & 7.93M & \textbf{62.9G}
  & \textbf{317} \\
\bottomrule
\end{tabular}
\end{wraptable}
The ablation also offers a possible explanation for part of the accuracy gap between TuBA and FourCastNet. Under frame-by-frame generation, TuBA's NRMSE rises to $0.426$, approaching FourCastNet's $0.442$, compared with $0.349$ when generating four frames jointly. This suggests that TuBA's ability to combine bidirectional interactions within each tensor with causal attention across tensors contributes to its accuracy advantage, while other architectural differences should also be considered. Yet even at single-frame granularity, autoregressive TuBA substantially outperforms autoregressive ViT with standard
self-attention,
suggesting that Tucker core attention also contributes to its higher accuracy.

\section{Conclusion and Future Work}
We introduced Tucker bottleneck attention (TuBA), which exploits the low-rank tensor structure of hidden representations to perform global token mixing within a learned Tucker subspace. By computing standard self-attention and its linear projections entirely within compact cores, TuBA reduces computational cost. Our analysis establishes its Tucker interpretation, recovers standard self-attention as a special case, and shows the increased expressivity of independently learned read and write factors. Its autoregressive extension combines bidirectional interactions within each core with causal attention across cores. Experiments on video prediction and global weather forecasting demonstrate improved prediction accuracy and reduced computation relative to standard self-attention, with favorable accuracy--efficiency trade-offs against efficient-attention alternatives and task-specific models. The results further show that low-rank Tucker cores and multi-frame generation can outperform full-rank attention and frame-by-frame generation, respectively. These findings support low-rank tensor structure as a useful inductive bias for sequence modeling and motivate further exploration of attention within learned representation subspaces. Currently, Tucker ranks are selected empirically, and our evaluation focuses on data with native multidimensional structure. Future work could investigate adaptive rank selection and extend TuBA to long one-dimensional sequences, where the choice of tensorization introduces an additional modeling decision.





\bibliography{iclr2027_conference}
\bibliographystyle{iclr2027_conference}
\clearpage

\appendix

\section{Shared Tucker subspaces of hidden representations}
\label{sec:shared-subspace}

We examine whether the hidden representations of a standard transformer admit a low-rank Tucker structure whose factor matrices are shared across samples. We extract hidden tensors $\ten{X}\in\mathbb{R}^{I_1\times\cdots\times I_D\times d}$ across the layers of the trained ViT baselines on Human3.6M and KittiCaltech and compare three Tucker representations. In all three, each sample has its own core, computed by projection onto the factors; they differ only in how the mode-wise factors are obtained. With shared factors, a single set of factors is fitted once to $128$ training samples and applied unchanged to $128$ held-out test samples. With per-sample factors, both the factors and the core are fitted to each test sample individually by HOSVD~\citep{delathauwer2000hosvd}, which serves as a reference for how well each sample can be represented at the given ranks. Random orthonormal factors of the same ranks serve as a control. We measure the relative reconstruction error $\|\ten{X}-\hat{\ten{X}}\|_F/\|\ten{X}\|_F$, averaged over the test samples, as a function of the core size $R/N$.

Figure~\ref{fig:shared-subspace} shows the results for three representative layers (0, 5, and 8), with dotted lines marking the core sizes used by TuBA in the main experiments; random factors are included as a control. Across layers, datasets, and core sizes, shared factors with per-sample cores closely approach the reconstruction accuracy of per-sample factors, with a slightly larger gap on KittiCaltech, whose driving scenes are more diverse. Thus, the underlying Tucker subspace is largely shared across samples: fixed mode-wise factors learned from training data capture most of the structure of unseen samples, while the sample-specific information is carried by a compact core. This observation motivates TuBA's design, which learns read and write factors shared across inputs and applies attention to the input-dependent core.

\begin{figure}[h]
    \centering
\includegraphics[width=\textwidth]{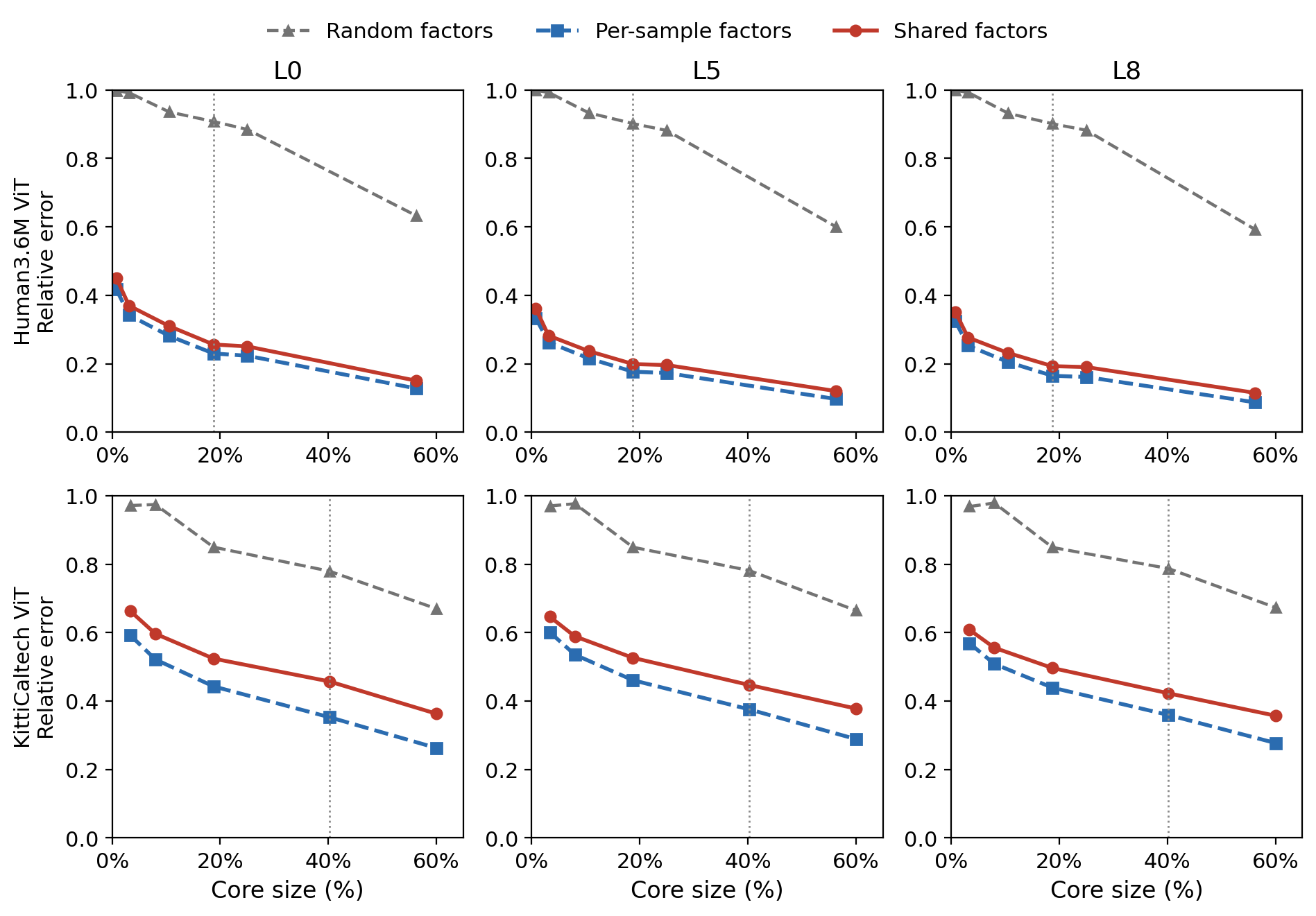}
    \caption{Tucker reconstruction error of held-out ViT hidden representations on Human3.6M (top) and KittiCaltech (bottom) at three representative layers. Shared factors with per-sample cores closely approach per-sample factors; dotted lines mark TuBA's core sizes. Random factors are included as a control.}
    \label{fig:shared-subspace}
\end{figure}

\section{Proof of Proposition~\ref{prop:tuba-tucker}}
\label{sec:tuba-tucker-proof}

\begin{proof}
For each token mode $n$, define
$\mat{Q}_n=\mat{A}_n^\top\mat{A}_n$.
Since $\mat{A}_n$ has full column rank, $\mat{Q}_n$ is
positive definite and
$\mat{U}_n=\mat{A}_n\mat{Q}_n^{-1}$ also has full column rank.
Consequently,
\[
\mat{U}_n^\top\mat{U}_n=\mat{Q}_n^{-1},
\qquad
\mat{U}_n^\dagger
=
(\mat{U}_n^\top\mat{U}_n)^{-1}\mat{U}_n^\top
=
\mat{A}_n^\top.
\]

To establish the least-squares statement, consider
\[
\underset{
\ten{G}\in\mathbb{R}^{r_1\times\cdots\times r_D\times d}
}{\operatorname{min}}
\left\|
\ten{X}
-
\ten{G}\times_1\mat{U}_1
\times_2\cdots\times_D\mat{U}_D
\right\|_F^2.
\]
Flatten the first $D$ modes with the mode-1 index varying
fastest, and let
$\mat{X}\in\mathbb{R}^{N\times d}$ and
$\mat{G}\in\mathbb{R}^{R\times d}$
be the resulting matrices, where
$N=\prod_{n=1}^D I_n$ and $R=\prod_{n=1}^D r_n$.
Define
\[
\mat{U}=\mat{U}_D\otimes\cdots\otimes\mat{U}_1,
\qquad
\mat{A}=\mat{A}_D\otimes\cdots\otimes\mat{A}_1.
\]
Reshaping preserves the Frobenius norm, so the objective
becomes $\|\mat{X}-\mat{U}\mat{G}\|_F^2$. Its unique
minimizer is
\[
\mat{G}_*
=
\mat{U}^\dagger\mat{X}
=
\left(
\mat{U}_D^\dagger\otimes\cdots\otimes\mat{U}_1^\dagger
\right)\mat{X}
=
\mat{A}^\top\mat{X}.
\]
Reshaping this matrix gives
\[
\ten{G}_*
=
\ten{X}\times_1\mat{A}_1^\top
\times_2\cdots\times_D\mat{A}_D^\top
=
\ten{C}.
\]

For the reconstruction, define
\[
\mat{P}_n
=
\mat{U}_n\mat{A}_n^\top
=
\mat{A}_n
(\mat{A}_n^\top\mat{A}_n)^{-1}
\mat{A}_n^\top.
\]
Each $\mat{P}_n$ is the orthogonal projector onto
$\operatorname{range}(\mat{A}_n)$.
Substituting the read core into the reconstruction yields
\[
\hat{\ten{X}}
=
\ten{C}\times_1\mat{U}_1
\times_2\cdots\times_D\mat{U}_D
=
\ten{X}\times_1\mat{P}_1
\times_2\cdots\times_D\mat{P}_D.
\]
In flattened form,
\[
\hat{\mat{X}}
=
\left(\mat{P}_D\otimes\cdots\otimes\mat{P}_1\right)
\mat{X}.
\]
The Kronecker product of these projectors is symmetric
and idempotent, and its range is the tensor-product
subspace determined by the read factors.
Hence, this reconstruction is the orthogonal projection
onto that subspace, with the channel mode unchanged.

Finally, for any transformed core $\ten{C}'$, the output
\[
\ten{Y}
=
\ten{C}'\times_1\mat{B}_1
\times_2\cdots\times_D\mat{B}_D
\]
is an exact Tucker representation, with the identity
factor on the channel mode understood.
For each token mode $n$, its mode-$n$ unfolding has the form
\[
\mat{Y}_{(n)}=\mat{B}_n\mat{Z}_n,
\]
where $\mat{Z}_n$ collects the transformed core and the
contractions along the remaining modes. Therefore,
\[
\operatorname{rank}(\mat{Y}_{(n)})
\leq \operatorname{rank}(\mat{B}_n)
\leq r_n,
\qquad n=1,\ldots,D.
\]
This final claim does not require full-column-rank
write factors.
\end{proof}

\section{Proof of Theorem~\ref{thm:expressivity}}
\label{sec:tuba-expressivity-proof}

\begin{proof}
Every tied layer is an untied layer with
$\mat{A}_n=\mat{B}_n$ for all $n$, so
$\mathcal{F}_{\mathrm{tied}}
\subseteq\mathcal{F}_{\mathrm{untied}}$.
We prove strict inclusion by constructing an untied
mapping that no tied layer can represent.

Consider a TuBA layer
\[
F(\mat{X})
=
\mat{B}\,\mathrm{MHSA}(\mat{A}^\top\mat{X}),
\]
with MHSA as in Eq.~\eqref{eq:attn}.
For a fixed input direction $\mat{Z}$ and scalar $t$,
the query--key scores at input $t\mat{Z}$ are of order
$t^2$. Hence, each core attention matrix satisfies
\[
\widetilde{\mat{S}}_i(t\mat{A}^\top\mat{Z})
=
\frac{1}{R}\vect{1}_R\vect{1}_R^\top+O(t^2).
\]
Define
\[
\mat{M}=\sum_{i=1}^H\mat{M}_i,
\qquad
\vect{a}=\mat{A}\vect{1}_R,
\qquad
\vect{b}=\mat{B}\vect{1}_R.
\]
The layer therefore satisfies
\[
F(t\mat{Z})
=
\frac{t}{R}\vect{b}\vect{a}^\top\mat{Z}\mat{M}
+O(t^3).
\]
For a tied layer, let $\mat{T}$ denote its common
Kronecker-product read and write matrix, and set
$\vect{c}=\mat{T}\vect{1}_R$.
Its derivative at the origin has the form
\[
DF_{\mathrm{tied}}(\mat{0})[\mat{Z}]
=
\frac{1}{R}\vect{c}\vect{c}^\top
\mat{Z}\mat{M}_{\mathrm{tied}}.
\]

Now choose untied factors such that
$\vect{a}$ and $\vect{b}$ are nonzero and noncollinear.
Such a choice is possible because $N>1$ implies
$I_n>1$ for at least one mode, and
\[
\vect{a}
=
(\mat{A}_D\vect{1}_{r_D})\otimes\cdots\otimes
(\mat{A}_1\vect{1}_{r_1}),
\qquad
\vect{b}
=
(\mat{B}_D\vect{1}_{r_D})\otimes\cdots\otimes
(\mat{B}_1\vect{1}_{r_1}).
\]
The column sums can be chosen nonzero and equal in all
other modes, and noncollinear in a mode with $I_n>1$.

Set the query and key projections to zero, making
every attention matrix exactly
$\vect{1}_R\vect{1}_R^\top/R$.
Choose the value and output projections so that
\[
\mat{M}=\mat{E}:=\vect{e}_1\vect{e}_1^\top,
\]
where $\vect{e}_1\in\mathbb{R}^d$ is the first coordinate
vector. This rank-one channel map is realizable by one
head, with all other heads contributing zero.
The resulting untied layer is the linear map
\[
F_{\mathrm{untied}}(\mat{X})
=
\frac{1}{R}\vect{b}\vect{a}^\top\mat{X}\mat{E}.
\]
Since this map is linear, its derivative is
\[
DF_{\mathrm{untied}}(\mat{0})[\mat{Z}]
=
\frac{1}{R}\vect{b}\vect{a}^\top\mat{Z}\mat{E}.
\]
Suppose a tied layer represents this same mapping.
Their derivatives at the origin must agree.
Applying them to
$\mat{Z}=\vect{x}\vect{e}_1^\top$ and taking the first
output channel gives
\[
\vect{b}\vect{a}^\top\vect{x}
=
\gamma\vect{c}\vect{c}^\top\vect{x}
\quad\text{for every }\vect{x}\in\mathbb{R}^N,
\qquad
\gamma
=
\vect{e}_1^\top\mat{M}_{\mathrm{tied}}\vect{e}_1.
\]
Thus,
\[
\vect{b}\vect{a}^\top
=
\gamma\vect{c}\vect{c}^\top.
\]
$\vect{b}\vect{a}^\top=\gamma\vect{c}\vect{c}^\top$ is symmetric, contradicting the noncollinearity of $\vect{a}$ and $\vect{b}$. Hence, $\mathcal{F}_{\mathrm{tied}}\subsetneq\mathcal{F}_{\mathrm{untied}}$.
\end{proof}

\section{TuBA computational cost details}
\label{sec:tuba-cost-details}

Consider a read contraction order $\pi=(\pi_1,\ldots,\pi_D)$. Before the
$k$th contraction, modes $\pi_1,\ldots,\pi_{k-1}$ have been reduced to
their ranks, while the remaining modes retain their ambient sizes.
Multiplication by $\mat{A}_{\pi_k}^{\top}$ therefore costs
\[
d\,I_{\pi_k}r_{\pi_k}\Big(\prod_{j<k}r_{\pi_j}\Big)\Big(\prod_{j>k}I_{\pi_j}\Big)
=
dN\,r_{\pi_k}\prod_{j<k}\frac{r_{\pi_j}}{I_{\pi_j}}
\;\leq\; dN\,r_{\pi_k},
\]
since $r_n\leq I_n$. Summing over modes, the read costs at most
$N\bar{R}d$ with $\bar{R}=\sum_{n=1}^{D}r_n$, for any contraction order.
The write applies the factors $\mat{B}_{\pi_k}$ in reverse order at the
same cost, so the read and write together cost $O(N\bar{R}d)$.
Contracting modes with smaller $r_n/I_n$ first further reduces the
exact cost.

Since $r_n\leq R$, we have $\bar{R}\leq DR$. Hence, for fixed $D$ and
$R=o(N)$, $N\bar{R}d=o(N^2d)$, $R^2d=o(N^2d)$, and $Rd^2=o(Nd^2)$, so
the TuBA cost $O(N\bar{R}d+R^2d+Rd^2)$ is asymptotically lower than
the $O(N^2d+Nd^2)$ cost of full attention.

For memory, every intermediate tensor in the read and write has at most $Nd$ entries, since each contraction only shrinks the tensor and each expansion recovers at most the ambient size; TuBA therefore requires $O(Nd+HR^2)$ activation memory, and $O(Nd)$ with FlashAttention.

\section{Implementation Details}
\label{sec:implement-details}
\subsection{Video-prediction implementation details}
\label{sec:video-details}

\paragraph{Tasks and metrics.}
Human3.6M uses $4{\to}4$ prediction at $256{\times}256{\times}3$, and KittiCaltech uses $10{\to}1$ prediction at $128{\times}160{\times}3$~\citep{tan2023openstl}; KittiCaltech models are trained on KITTI and evaluated on the $3{,}631$-clip Caltech Pedestrian test set. We report mean squared error (MSE), mean absolute error (MAE), structural similarity index measure (SSIM)~\citep{wang2004ssim}, peak signal-to-noise ratio (PSNR), and learned perceptual image patch similarity (LPIPS)~\citep{zhang2018lpips}. MSE and MAE are summed over pixels and channels, giving scales of $196{,}608$ for Human3.6M and $61{,}440$ for KittiCaltech. Parameter counts and FLOPs are reported per input clip.

\paragraph{Tokenization.}
The models in Tables~\ref{tab:human36m} and~\ref{tab:kitticaltech} share the SimVP encoder--decoder~\citep{gao2022simvp}, with four convolutional blocks in both the encoder and the decoder, $64$ hidden channels, and $3{\times}3$ convolutions, and differ only in the token mixer between them. The encoder maps the input sequence to a latent tensor, and $2{\times}2$ latent patches define the tokens, yielding token tensors of size $4{\times}32{\times}32$ ($4{,}096$ tokens) on Human3.6M and $10{\times}16{\times}20$ ($3{,}200$ tokens) on KittiCaltech. All attention-based models add a learned absolute positional embedding, with one vector per token, before the first block.

\paragraph{Architecture.} On each benchmark, TuBA, ViT, linear attention, Linformer, and Nystr\"{o}m attention use the same transformer block, a pre-norm attention residual followed by a pre-norm SwiGLU residual, and differ only in the attention operation. The block has width $256$, eight heads of size $32$, and dropout $0.1$, with SwiGLU hidden size $682$ on Human3.6M and $256$ on KittiCaltech. TuBA uses Tucker ranks $(3,16,16)$ on Human3.6M and $(9,11,13)$ on KittiCaltech; ViT applies full attention over all tokens; Linformer projects the keys and values to length $128$; and Nystr\"{o}m attention uses $64$ landmarks. SimVP instead uses its original IncepU translator with width $256$.

\paragraph{Optimization.}
All models are trained in \texttt{fp32} using AdamW, weight decay $0.01$, and a cosine schedule decaying to $10^{-6}$, without data augmentation, gradient clipping, or drop path. Human3.6M models are trained for $50$ epochs with batch size $4$, and KittiCaltech models for $100$ epochs with batch size $8$.

\paragraph{Throughput.}
FPS is measured at inference with batch size $4$ in \texttt{fp32}. FPS$^*$ uses mixed precision, with a fused FlashAttention kernel~\citep{dao2022flashattention} when available and \texttt{bf16} arithmetic otherwise.
\begin{table}[h]
\centering
\small
\caption{Effect of Tucker ranks on TuBA's accuracy on KittiCaltech.}
\label{tab:rank-ablation}
\begin{tabular}{@{}r l rrrrr@{}}
\toprule
$R$ & Ranks & MSE$\downarrow$ & MAE$\downarrow$ & SSIM$\uparrow$ & PSNR$\uparrow$ & LPIPS$\downarrow$ \\
\midrule
720  & $(9,8,10)$   & 170.54 & 1737.3 & 0.8889 & 26.74 & 0.0881 \\
891  & $(9,9,11)$   & 165.20 & 1665.5 & 0.8881 & 26.97 & 0.0898 \\
972  & $(9,9,12)$   & 172.03 & 1757.1 & 0.8855 & 26.66 & 0.0972 \\
1080 & $(9,10,12)$  & 162.53 & 1641.4 & 0.8948 & 27.15 & 0.0836 \\
1188 & $(9,11,12)$  & 158.15 & 1654.2 & 0.8957 & 27.14 & 0.0832 \\
1287 & $(9,11,13)$  & \textbf{156.07} & 1641.7 & \textbf{0.8963} & \textbf{27.25} & \textbf{0.0772} \\
1620 & $(9,12,15)$  & 158.43 & \textbf{1608.7} & 0.8955 & 27.17 & 0.0820 \\
1728 & $(9,12,16)$  & 160.17 & 1624.8 & 0.8950 & 27.15 & 0.0841 \\
3200 & $(10,16,20)$ & 176.98 & 1770.7 & 0.8720 & 26.56 & 0.1132 \\
\bottomrule
\end{tabular}
\end{table}
\subsection{Rank ablation details}
\label{sec:rank-ablation}

Table~\ref{tab:rank-ablation} reports the full results of the Tucker rank sweep on KittiCaltech. All models share the same backbone as in Appendix~\ref{sec:video-details}, with the $10{\times}16{\times}20$ token tensor, and differ only in the Tucker ranks; the core size $R$ is their product. The temporal rank is fixed at $9$ for all reduced cores, so the sweep varies the two spatial ranks, while $R=3200$ corresponds to the full-rank setting with no reduction. The core $R=1287$ attains the best MSE, SSIM, PSNR, and LPIPS, while the next larger core, $R=1620$, attains the lowest MAE.

\subsection{Weather Forecasting Implementation Details}
\label{sec:weather-detials}
\paragraph{Data and forecasting task.}
We use ERA5~\citep{hersbach2020era5} on a $128{\times}256$ latitude--longitude grid at $1.40625^\circ$ resolution, subsampled from the hourly archive at six-hour intervals. Each example contains 12 input frames and 12 target frames. Models are trained on 2011--2015 and evaluated on 2017--2018. The state contains 25 channels: geopotential ($z$), relative humidity ($r$), temperature ($t$), and horizontal wind components ($u$, $v$), each at $50$, $500$, $850$, and $1000\,\mathrm{hPa}$, together with five surface or single-level fields: 2\,m temperature ($t_{2\mathrm{m}}$), 10\,m wind components ($u_{10}$, $v_{10}$), total cloud cover ($\mathrm{tcc}$), and top-of-atmosphere incident solar radiation ($\mathrm{tisr}$).

\paragraph{Tokenization and architecture.}
TuBA and the attention baselines use non-overlapping $8{\times}8$ spatial patches, producing a $16{\times}32$ token grid per frame and a $12{\times}16{\times}32$ token tensor ($6{,}144$ tokens) per 12-frame context, and reconstruct predictions with a linear output head followed by unpatching, without a convolutional encoder--decoder. All attention-based models, including TuBA, add a learned absolute positional embedding, with one vector per token, before the first block. TuBA, ViT, linear attention, Linformer, Nystr\"{o}m attention, and the autoregressive TuBA and ViT use the same transformer block as in the video experiments, with width $256$, four heads of size $64$, SwiGLU hidden size $682$, and no dropout, and differ only in the attention operation. One-shot TuBA uses Tucker ranks $(10,8,16)$ along the temporal, latitudinal, and longitudinal token axes; Linformer projects the keys and values to length $256$, and Nystr\"{o}m attention uses $64$ landmarks. SimVP retains its convolutional encoder--decoder and IncepU translator with spatial width $64$ and translator width $176$.

\paragraph{Autoregressive generation.}
Autoregressive TuBA uses block-causal attention with a granularity of four frames per tensor, so the 12-frame context forms three tensors, each with Tucker ranks $(3,8,16)$. Training uses teacher forcing over the concatenated context and target sequence in a single forward pass. At inference, the model generates three successive four-frame tensors, extending the sequence from 12 to 24 frames, with a per-layer key--value cache that processes the context once and is then updated with each generated tensor. Learned token positional embeddings span one four-frame tensor and are shared across tensors.

\paragraph{Evaluation metrics.}
For each channel $c$, we compute $\mathrm{NRMSE}_c=\mathrm{RMSE}_c/\sigma_c$, where $\sigma_c$ is the standard deviation estimated from the training data. The aggregate score is the mean of the 25 channel-wise values. FLOPs are measured per complete 12-frame forecast, and throughput is the best measured number of complete forecasts per second at batch size eight.

\subsection{Training memory and wall-clock time}
\label{sec:memory-walltime}

Table~\ref{tab:memory-walltime} reports the peak GPU memory of one \texttt{fp32} training step and the average wall-clock time per training epoch on a single NVIDIA H100 GPU. Memory is the total held by the training process at that peak, read from \texttt{nvidia-smi}, and therefore includes live tensors, the reserved allocator cache, and the CUDA context. TuBA reduces training memory relative to standard self-attention by $3.5$--$5.1\times$ and trains $1.7$--$4.3\times$ faster per epoch across the three benchmarks, while remaining comparable to the efficient-attention baselines.

\begin{table}[h]
\centering
\small
\caption{Peak training GPU memory (GiB) and wall-clock time per epoch (minutes) on a single H100 GPU in \texttt{fp32}. Batch sizes are 4 for Human3.6M and 8 for KittiCaltech and ERA5.}
\label{tab:memory-walltime}
\begin{tabular}{@{}l rr rr rr@{}}
\toprule
& \multicolumn{2}{c}{Human3.6M} & \multicolumn{2}{c}{KittiCaltech} & \multicolumn{2}{c}{ERA5} \\
\cmidrule(lr){2-3} \cmidrule(lr){4-5} \cmidrule(l){6-7}
Method & Memory & Time/epoch & Memory & Time/epoch & Memory & Time/epoch \\
\midrule
TuBA        & \textbf{10.22} & 29.3 & 15.05 & 1.68 & \textbf{11.99} & 7.50 \\
ViT         & 38.72 & 95.7 & 52.28 & 7.23 & 61.61 & 12.68 \\
Linear      & 10.61 & \textbf{23.7} & 10.86 & \textbf{1.10} & 13.35 & 7.67 \\
Linformer   & 10.47 & 25.2 & \textbf{10.04} & 1.30 & 13.42 & 7.37 \\
Nystr\"{o}m & 10.96 & 29.5 & 12.54 & 1.57 & 12.96 & \textbf{7.31} \\
\bottomrule
\end{tabular}
\end{table}

\end{document}